\documentclass[letterpaper, 10 pt, journal, twoside]{IEEEtran}

\usepackage{graphicx}
\usepackage{amsmath}
\usepackage{amssymb}
\usepackage{amsfonts}
\usepackage{amsthm}
\usepackage[most]{tcolorbox}
\usepackage{xcolor}
\usepackage{booktabs, multirow} 
\usepackage{soul}
\usepackage{xcolor,colortbl}
\usepackage{changepage,threeparttable} 
\usepackage{hyperref}
\usepackage[noadjust]{cite}

\usepackage{tikz}
\usepackage[linesnumbered,ruled,vlined]{algorithm2e}

\usepackage{txfonts}
\usepackage{newtxtext,newtxmath}
\usepackage{newtxtt}
\usepackage{wrapfig}

\usetikzlibrary{tikzmark}

\definecolor{selectiong}{HTML}{FCCC99}
\definecolor{negationg}{HTML}{6A01FF}
\definecolor{graycirc}{HTML}{657687}
\definecolor{bluecirc}{HTML}{B1DDF0}
\definecolor{redcirc}{HTML}{F86801}

\newcommand{\F}{\mathbf{F}}
\newcommand{\G}{\mathbf{G}}
\newcommand{\U}{\mathbf{U}}
\newcommand{\Predicate}{P_{\theta}}
\newcommand{\PredicateWithIndex}[1]{P^{#1}_{\theta_#1}}
\newcommand{\argmax}{\mathop{\mathrm{argmax}}}
\newtheorem{theorem}{Theorem}
\newcommand{\PredicateSet}{\mathcal{P}}
\newcommand{\LogicStucture}{\mathcal{L}}

\newtcbox{\entoure}[1][red]{on line,
    arc=3pt,
    colback=#1!10!white, 
    colframe=#1!90!black, 
    before upper={\rule[-3pt]{0pt}{10pt}},boxrule=1pt,
    boxsep=0pt,left=2pt,right=2pt,top=1pt,bottom=.5pt}

\newcommand{\selectiongate}{\tikzmarknode[draw,inner sep=2pt,rounded corners,fill=selectiong]{A}{selection gate}}
\newcommand{\negationgate}{\tikzmarknode[draw,inner sep=2pt,rounded corners,fill=negationg,text=white]{A}{negation gate}}

\newcommand*\circled[1]{\tikz[baseline=-0.5ex]{
        \node[shape=circle,draw,inner sep=2pt,fill=#1] (char) {};}}

\usepackage{balance}

\title{\LARGE \bf
FLoRA: A \underline{F}ramework for \underline{L}earning Sc\underline{o}ring \underline{R}ules in \underline{A}utonomous Driving Planning Systems}

\author{
    Zikang Xiong, Joe Eappen, and Suresh Jagannathan 
\thanks{ Authors are with the Computer Science Department, Purdue University. 
{\tt\small \{xiong84,jeappen,suresh\}@cs.purdue.edu}}
}

\begin{document}
\maketitle

\begin{abstract}
    In autonomous driving systems, motion planning is commonly implemented as a two-stage process: first, a trajectory proposer generates multiple candidate trajectories, then a scoring mechanism selects the most suitable trajectory for execution. For this critical selection stage, rule-based scoring mechanisms are particularly appealing as they can explicitly encode driving preferences, safety constraints, and traffic regulations in a formalized, human-understandable format. However, manually crafting these scoring rules presents significant challenges: the rules often contain complex interdependencies, require careful parameter tuning, and may not fully capture the nuances present in real-world driving data.
    This work introduces FLoRA, a novel framework that bridges this gap by learning interpretable scoring rules represented in temporal logic. Our method features a learnable logic structure that captures nuanced relationships across diverse driving scenarios, optimizing both rules and parameters directly from real-world driving demonstrations collected in NuPlan. Our approach effectively learns to evaluate driving behavior even though the training data only contains positive examples (successful driving demonstrations). Evaluations in closed-loop planning simulations demonstrate that our learned scoring rules outperform existing techniques, including expert designed rules and neural network scoring models, while maintaining interpretability.
    This work introduces a data-driven approach to enhance the scoring mechanism in autonomous driving systems, designed as a plug-in module to seamlessly integrate with various trajectory proposers.  Our video and code are available on \href{https://xiong.zikang.me/FLoRA/}{xiong.zikang.me/FLoRA/}.
\end{abstract}

\section{Introduction}
\label{sec:intro}

\begin{figure}[ht!]
    \centering
    \includegraphics[width=\linewidth]{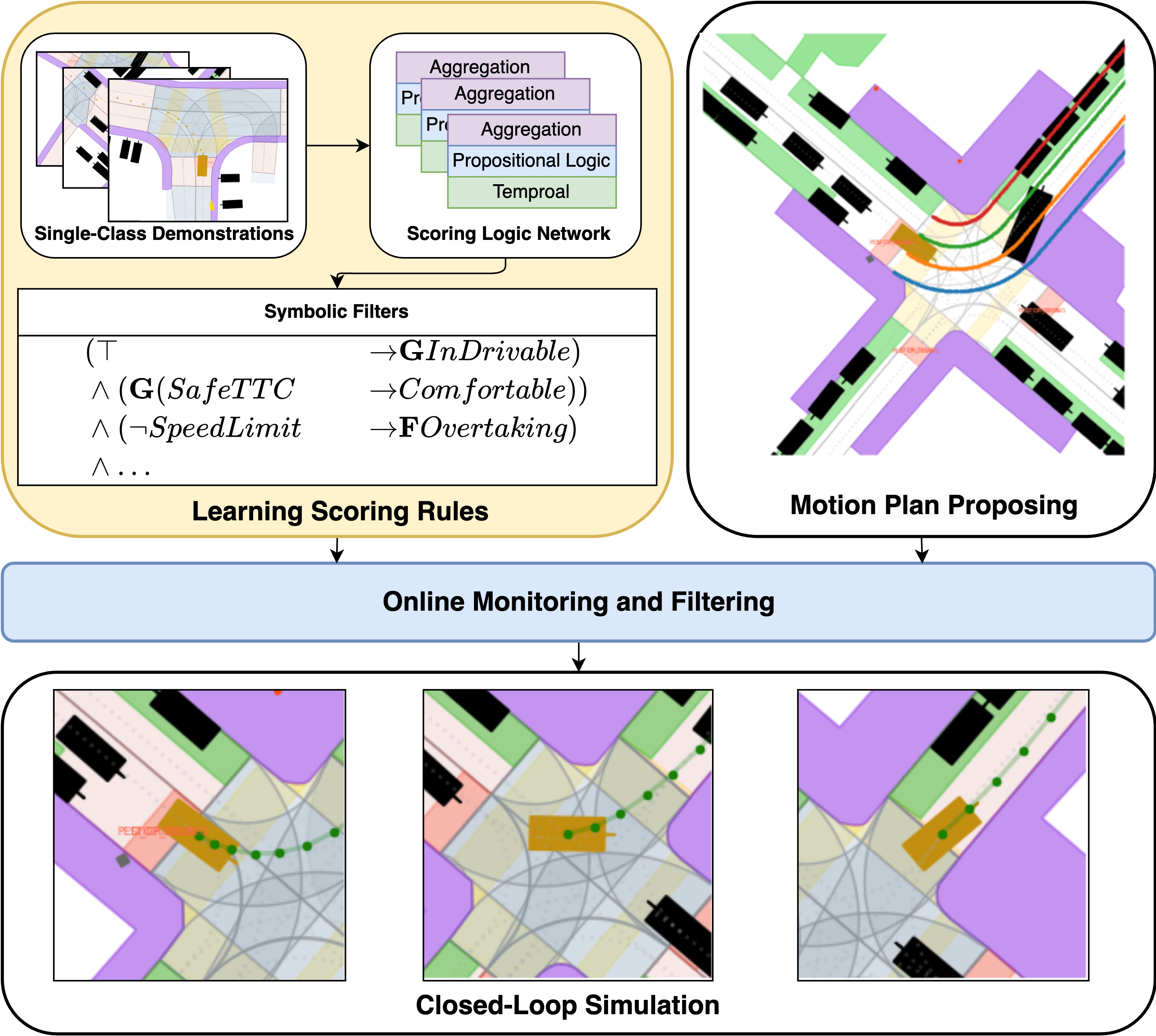}
    \caption{ This figure illustrates our framework for scoring and selecting trajectories in autonomous driving systems. Modern autonomous driving planners typically follow a propose-selection paradigm, where multiple candidate trajectories are first generated and then filtered through a scoring mechanism. As shown in the \textbf{Motion Plan Proposing} block, multiple trajectories (colored lines) are proposed as potential future paths for the autonomous vehicle. These candidates need to be evaluated and ranked to select the most suitable trajectory for execution. Our key contribution, the \textbf{Learning Scoring Rules} block, demonstrates how we learn interpretable scoring rules from human driving demonstration from NuPlan \cite{Karnchanachari2024TowardsLP}. Instead of manually crafting rules or using black-box scoring models, we develop a Scoring Logic Network (SLN) that automatically learns temporal logic rules from data (specific rules illustrated in Sec.~\ref{sec:logic_rules_discovered}). These learned rules are then deployed in the \textbf{Online Monitoring and Filtering} block, which continuously evaluates the proposed trajectories at 20 Hz following NuPlan's log frequency \cite{Karnchanachari2024TowardsLP} during operation. In the scene visualization, the ego vehicle is shown in gold, while other vehicles are represented in black, the curbs are marked in purple, and the driveable areas (lanes and intersections) are marked as pink and blue, respectively. Among the proposed trajectories, the green trajectory receives the highest score as it successfully maintains a safe distance from the curbs, while exhibiting appropriate curvature and comfort characteristics. All other colored trajectories are not selected for violating one or more learned rules, as detailed in Sec.~\ref{sec:logic_rules_discovered}. During the monitoring, following the setting in Planning Driver Model (PDM) \cite{Dauner2023CORL}, we assume that the future 4-second trajectories of other cars are known. The selected trajectory is then executed in the \textbf{Closed-Loop Simulation} block.
    }
    \label{fig:motivation_example}
\end{figure}

Modern autonomous driving systems typically produce multiple potential plans \cite{Dauner2023CORL,hu2023planning,phan2022driving,jiang2022efficient,chen2024end}, as this parallel approach offers several key advantages: it allows the system to consider different driving modalities (such as aggressive or conservative behaviors), accounts for future uncertainties, and provides redundancy in case certain paths become infeasible. These generated plans then need to be evaluated through a scoring mechanism to select the most suitable one for execution.

The importance of effective scoring becomes even more apparent in complex autonomous driving systems, particularly in end-to-end approaches \cite{hu2023planning,chen2024end}. These systems often utilize large, intricate neural networks that incorporate elements of randomness, such as dropout or sampling from probability distributions. While powerful, such characteristics can make it challenging to predict the system's behavior consistently \cite{chen2024end}. By implementing interpretable scoring rules as a final evaluation layer, we can assess these generated plans against clear, understandable criteria, thereby introducing much-needed predictability and reliability to these complex systems. In essence, scoring techniques serve as a critical bridge, connecting the raw outputs of planning algorithms to the final, executable plans. This additional evaluation step helps mitigate the uncertainties inherent in complex planning systems, significantly enhancing the safety, efficiency, and overall performance of autonomous vehicles as they navigate through our highways and cities.

With real-world driving data available in datasets like NuPlan \cite{Karnchanachari2024TowardsLP}, most current learning methods focus on directly learning motion plan proposers rather than learning interpretable scoring mechanisms to evaluate these plans. We instead focus on learning scoring rules represented in temporal logic for evaluating autonomous driving plans, which assess and rank plans generated by motion plan proposers. These scoring rules capture the latent relationships between various driving rules and constraints; for example, if a vehicle has a safe time-to-collision with surrounding vehicles, it should always be subject to all comfort constraints. By applying these rules to the output of a motion planner, we can score and select desirable plans, ensuring that the planned paths adhere to safety standards and traffic regulations while maintaining optimal performance. Figure~\ref{fig:motivation_example} illustrates the learning process and how we might apply scoring rules.

In practice, building these scoring rules presents several significant challenges. First, \emph{the latent relationships and dependencies among various rules are often non-trivial}. For instance, while a vehicle is generally not permitted to exceed the speed limit, exceptions may exist in specific scenarios such as overtaking another vehicle. These nuanced dependencies make it challenging to create a comprehensive set of rules that account for all possible situations. Second, \emph{determining the optimal parameters for rules is a complex task}. For example, establishing appropriate thresholds for safe time-to-collision, comfortable acceleration, or acceptable steering angle requires careful consideration of multiple factors. These parameters must balance safety concerns with the need for efficient and smooth vehicle operation. Third, \emph{available demonstration data typically only showcases correct behavior and lacks sufficient examples of incorrect actions} \cite{Karnchanachari2024TowardsLP,akhauri2020enhanced,chen2024end}. Having only single-class (i.e., correct-behavior only) training data poses a significant challenge for learning scoring models, as they must learn to distinguish between acceptable and unacceptable behaviors without a sufficient number of explicit negative samples.

Our approach addresses these challenges through three interconnected key ideas. First, to capture latent relationships among driving rules, we introduce a learnable logic structure that seamlessly integrates temporal and propositional logic. Our structure can represent nuanced decisions such as when it is appropriate to exceed the speed limit for a safe overtaking maneuver. Second, we tackle the challenge of parameter optimization by letting the data speak for itself. Rather than relying on manual tuning, our system learns optimal rule parameters directly from driving demonstrations, leveraging the fully differentiable logic structure used to represent rules. Third, we overcome the limitation of learning from only positive examples through a novel regularization-constrained optimization framework that simultaneously rewards correct demonstration behaviors and restricts the space of acceptable ones.

We evaluate the efficacy of our learned scoring rules in NuPlan \cite{Karnchanachari2024TowardsLP} closed-loop simulations. These results demonstrate that our learned rules can effectively score and select desirable plans, outperforming both expert-crafted rules and neural network-based approaches when considering both interactive and non-interactive scenarios. We further show that the learned rules perform consistently well across different proposers, including PDM \cite{Dauner2023CORL}, PDM-Hybrid \cite{Dauner2023CORL}, ML-Prop \cite{Karnchanachari2024TowardsLP}, and a rule-based Acceleration-Time (AT) sampler \cite{jiang2022efficient}. In summary, our contributions are as follows:
\begin{itemize}
    \item We propose a novel learnable logic structure that discovers and captures latent relationships represented in temporal logic.
    \item We introduce a data-driven approach to optimize rule parameters, enabling the system to learn effective scoring rules from driving demonstrations.
    \item We present an optimization framework that allows the system to learn rules and parameters from human driving demonstrations, without requiring unsafe and rare accident examples.
    \item We demonstrate the effectiveness of our approach in scoring and selecting desirable plans via NuPlan closed-loop simulation, outperforming expert-crafted rules \cite{Dauner2023CORL} and neural network-based approaches \cite{jiang2022efficient} across various scenarios and proposers.
\end{itemize}
\section{Related Work}
Scoring models are widely used in autonomous driving systems to evaluate the safety, efficiency, and comfort of planned trajectories \cite{Dauner2023CORL,Karnchanachari2024TowardsLP,hu2023planning,xiao2021rule,wishart2020driving,junietz2018criticality,hekmatnejad2019encoding,Deng2023TARGETAS,Zhou2023SpecificationBasedAD,houston2021one,riedmaier2020survey,weng2020model,vasudevan2024planningadaptiveworldmodels,phan2022driving,huang2021driving,kuderer2015learning,wulfmeier2017large,silver2012learning,jiang2022efficient}. This evaluation becomes particularly crucial in end-to-end autonomous driving systems, where behavioral uncertainties arise from multiple sources: neural network stochasticity (e.g., dropout, distributional sampling)~\cite{hu2023planning}, environmental condition variations~\cite{hu2023planning}, and complex multi-agent interaction modeling~\cite{chen2024end}.

Existing scoring models can be broadly categorized into rule-based and learning-based models. Rule-based scoring models \cite{Dauner2023CORL,Karnchanachari2024TowardsLP,hu2023planning,xiao2021rule,wishart2020driving,junietz2018criticality,hekmatnejad2019encoding,Deng2023TARGETAS,Zhou2023SpecificationBasedAD,houston2021one,riedmaier2020survey} leverage expert knowledge and domain-specific rules, offering interpretability but lacking data-driven adaptability. In contrast, learning-based scoring models, typically approximated with deep neural network, \cite{weng2020model,vasudevan2024planningadaptiveworldmodels,phan2022driving,huang2021driving,kuderer2015learning,wulfmeier2017large,silver2012learning,jiang2022efficient} are data-driven and can capture complex patterns, but often sacrifice interpretability, which can be intimidating for safety-critical applications.

In the broader context of temporal logic rule learning \cite{bartocci2022survey}, some researchers have explored learning interpretable rules from demonstrations. However, the current literature has not explored the potential of learning interpretable rules directly from driving data, in the context of autonomous driving systems.

\section{Preliminaries}
\label{sec:problem_formulation}

In this section, we formulate the key technical components and our objectives.

\paragraph{Predicate $\Predicate$}
At a certain time point, given all environment information $\mathcal{E} = (M, I, A)$ and a motion plan $\tau$, the differentiable predicate $\Predicate$ is defined as: $\Predicate : (\mathcal{E} \times \tau) \rightarrow [-1, 1]$, which evaluates driving conditions and the ego car's motion plan \footnote{The ego car refers to the vehicle being controlled in a driving scenario.}. Here, $M$ represents the HD map API that provides the functionality to query a drivable area, lane information, routing information and occupancy grid; $I = \{i_1,...,i_K\}$ represents the set of traffic light states, where each $i_k \in \{\mathit{red, yellow, green}\}$; $A = \{a_1,...,a_N\}$ represents the set of other agents, where each $a_i = \{(x_t, y_t, v_t)\}_{t=0}^{T}$ contains the agent's trajectory for the next 4 seconds (i.e., $T = 80$, $20Hz \times 4 s$). The motion plan $\tau = \{(x_t, y_t, v_t)\}_{t=0}^{T}$  consisting of a sequence of positions $(x_t, y_t)$ and speeds $v_t$, where higher-order derivatives (e.g., acceleration, jerk) can be computed through numerical differentiation when needed.
$\Predicate$ maps $\mathcal{E}$ and $\tau$ to a truth confidence value in $[-1,1]$, where $\theta$ are the predicate parameters. Each predicate $\Predicate$ is parameterized by $\theta$, which defines thresholds or constraints specific to that predicate (e.g., safe time-to-collision threshold, comfortable acceleration bounds). When designing the predicate, we ensure that the gradient $\nabla_{\theta} \Predicate$ exists and can be computed. The sign of $\Predicate$ indicates truth. $\Predicate < 0$ implies False; $\Predicate > 0$ implies True. The absolute value $|\Predicate|$ indicates the degree of confidence.
Similar to most existing work \cite{bartocci2022survey}, we explicitly design the predicates and focus this paper on learning logical connections and parameters assuming a given set of predicates. With the predicate defined, we can now move on to discussing how these predicates are combined into logical formulas.

\paragraph{Formula $\mathcal{L}$}
Given a differentiable predicate set $\PredicateSet = \{\PredicateWithIndex{1}, \PredicateWithIndex{2}, \ldots, \PredicateWithIndex{n}\}$, we introduce a $\mathtt{LTL}_f$ logic space \cite{LTLf} (LTL over finite traces) that includes compositions of predicates from $\PredicateSet$ and logic operators. The logic formula  $\mathcal{L}$ can be generated from the following grammar:
\begin{align}
    \mathcal{L} := \Predicate \mid \G\ \mathcal{L} \mid \F\ \mathcal{L} \mid \lnot \mathcal{L} \mid \mathcal{L} \land \mathcal{L}' \mid \mathcal{L} \lor \mathcal{L}'
    \label{eq:logic_formula_syntax}
\end{align}
where $\Predicate \in \mathcal{P}$ is a differentiable predicate, $\G$ and $\F$ are temporal operators representing ``globally'' and ``finally'' respectively, $\lnot$ is logical negation, $\land$ is logical and, and $\lor$ is logical or.  Like most existing work \cite{bartocci2022survey}, the strong ``Until'' ($\mathcal{L}\ \U\ \mathcal{L}'$) is not included because it can be represented using existing logic operators ($\F\ \mathcal{L}' \wedge \G(\mathcal{L} \lor \mathcal{L}')$). Having established the syntax for our logic formulas, we now need a way to evaluate them quantitatively.

\paragraph{Quantitative Evaluation of Formula}
Given a finite input sequence $S = \{(\mathcal{E}_t, \tau_t)\}_{t=0}^T$ sampled at different time points, up to a bounded time $T$, we can evaluate the logic formula $\mathcal{L}$ quantitatively using a set of min and max operators\cite{fainekos2009robustness, deshmukh2017robust}. This evaluation maps the sequence $S$ to a value in $[-1, 1]$, denoted as
$
    \mathcal{L}(S; \boldsymbol{\theta}) \rightarrow [-1, 1].
$
Here, $\boldsymbol{\theta}$ represents all the predicates' parameters.
Specifically, we define the quantitative evaluation of an atomic predicate $\Predicate$ at time $t$ as $\Predicate(\mathcal{E}_t, \tau_t) \in [-1, 1]$. In \eqref{eq:temporal_operators}, the temporal logic operators $\G$ and $\F$ evaluate the formula $\mathcal{L}$ over the entire sequence from time $t$ onwards. $\G \mathcal{L}$ (globally) returns the minimum value of $\mathcal{L}$ over all future time points, which ensures the property holds throughout the sequence if $\G \mathcal{L}$ evaluates to a positive value. $\F \mathcal{L}$ (finally) returns the maximum value, indicating the property is satisfied at least once in the future. The evaluation function $\rho$ is defined as:
\begin{equation}
    \begin{aligned}
        \rho(\G \mathcal{L}, t) & = \min_{t' \geq t} \rho(\mathcal{L}, t') \quad & \rho(\F \mathcal{L}, t) & = \max_{t' \geq t} \rho(\mathcal{L}, t')
    \end{aligned}
    \label{eq:temporal_operators}
\end{equation}
For single time point evaluation, the logical operators and ($\land$), or ($\lor$), and not ($\lnot$) are defined using min, max, and negation operations:
\begin{equation}
    \begin{aligned}
         & \rho(\mathcal{L} \land \mathcal{L}', t) & = & \min\{\rho(\mathcal{L}, t), \rho(\mathcal{L}', t)\} \\
         & \rho(\mathcal{L} \lor \mathcal{L}', t)  & = & \max\{\rho(\mathcal{L}, t), \rho(\mathcal{L}', t)\} \\
         & \rho(\lnot \mathcal{L}, t)              & = & -\rho(\mathcal{L}, t)
    \end{aligned}
    \label{eq:fol_operators}
\end{equation}
All the operations defined by $\rho$ are differentiable, which enables the use of backpropagation in the learning process.  In practice, we use softmin and softmax to approximate min and max operators for a smooth gradient \cite{Leung2020BackpropagationTS}. With the evaluation framework in place, we can now define our overall objective for learning optimal driving rules. For simplicity, we define $\rho(\cdot):= \rho(\cdot, 0)$, meaning evaluate from the initial of input sequence.

\paragraph{Objective}
Our objective is twofold: (1) learn the optimal logic formula $\mathcal{L}^*$, and (2) optimize the parameters $\boldsymbol{\theta}$ of the predicates, which characterize the demonstration data accurately.
Formally, we aim to solve the following problem:
\begin{align}
    \mathcal{L}^*, \boldsymbol{\theta}^* = \argmax_{\mathcal{L} \in \Omega_\PredicateSet, \boldsymbol{\theta}} \mathbb{E}_{S \sim \mathcal{D}^+} [\mathcal{L}(S; \boldsymbol{\theta})]
    \label{eq:objective}
\end{align}
where $\mathcal{D}^+$ represents driving demonstrations, which consist \textit{solely} of correct demonstrations that represent ideal driving behaviors.
\section{Approach}
\label{sec:approach}

This section presents our approach to learning interpretable driving rules from demonstrations. We begin by introducing the concept of condition-action pairs in Sec.~\ref{sec:condition-action-pair}, which forms the foundation of our rule representation. Sec.~\ref{sec:learnable_logic_structure} introduces the core of our method: the learnable logic structure. Here, we explain its components, analyze its capabilities, and describe how we extract and simplify rules from it. Finally, Sec.~\ref{sec:learning_from_single_class_demonstration} introduces our novel regularization techniques to overcome the limitations of learning from positive-only examples.

\subsection{Condition-Action Pair}
\label{sec:condition-action-pair}
We consider rules that consist of conditions and expected actions. For instance, one such rule might require that the ego car eventually stop when approaching a stop sign. This pattern of condition-action pairing extends to countless driving situations. Thus, we focus on effectively learning and reducing driving rules to condition-action pairs.

Predicates are the basic unit of our rules. We categorize our predicates into two types: condition predicates $\bar{\PredicateSet} = \{\bar{\PredicateWithIndex{1}}, \bar{\PredicateWithIndex{2}}, \ldots, \bar{\PredicateWithIndex{n}}\}$ and action predicates $\dot{\PredicateSet} = \{\dot{\PredicateWithIndex{1}}, \dot{\PredicateWithIndex{2}}, \ldots, \dot{\PredicateWithIndex{m}}\}$. Condition predicates evaluate traffic conditions (e.g., if approaching a stop sign), while action predicates assess the motion plan (e.g. if the ego car is stopped). Given
\begin{equation}
    \begin{aligned}
        \mathit{condition} & := \bar{\Predicate} \mid \G\ condition \mid \F\ \mathit{condition} \mid \lnot \mathit{condition} \\
        \mathit{action}    & := \dot{\Predicate} \mid \G\ \mathit{action} \mid \F\ \mathit{action} \mid \lnot \mathit{action}
    \end{aligned}
\end{equation}
we extract and simplify the learned logic formula $\mathcal{L}$ to a set of propositional rules of the form:
\begin{equation}
    \bigodot_{i=1}^m \left(\bigwedge_{j=1}^n condition_j \rightarrow action_i\right)
    \label{eq:condition_action_pair}
\end{equation}
where $\bigodot$ denotes that this condition-action pairs are connected by $\land$ or $\lor$ operators. The conjunction of conditions ($\bigwedge_{j=1}^n condition_j$) allows for more precise and specific criteria to be defined for each action, thereby describing precisely when the action should be allowed. Building upon this foundation, we introduce a learnable logic structure to represent and learn these condition-action pairs.

\subsection{Learnable Logic Structure}
\label{sec:learnable_logic_structure}

The learnable logic structure $\bar{\mathcal{L}}$ is a directed acyclic computation graph that represents a compositional logic formula. It consists of three types of layers: \textbf{Temporal, Propositional}, and \textbf{Aggregation}. These layers are interconnected through learnable gates that determine the flow and combination of logical operations. An example of this structure is shown in Fig.~\ref{fig:learnable_logic_structure}.

\begin{figure}[!htp]
    \centering
    \includegraphics[width=\linewidth]{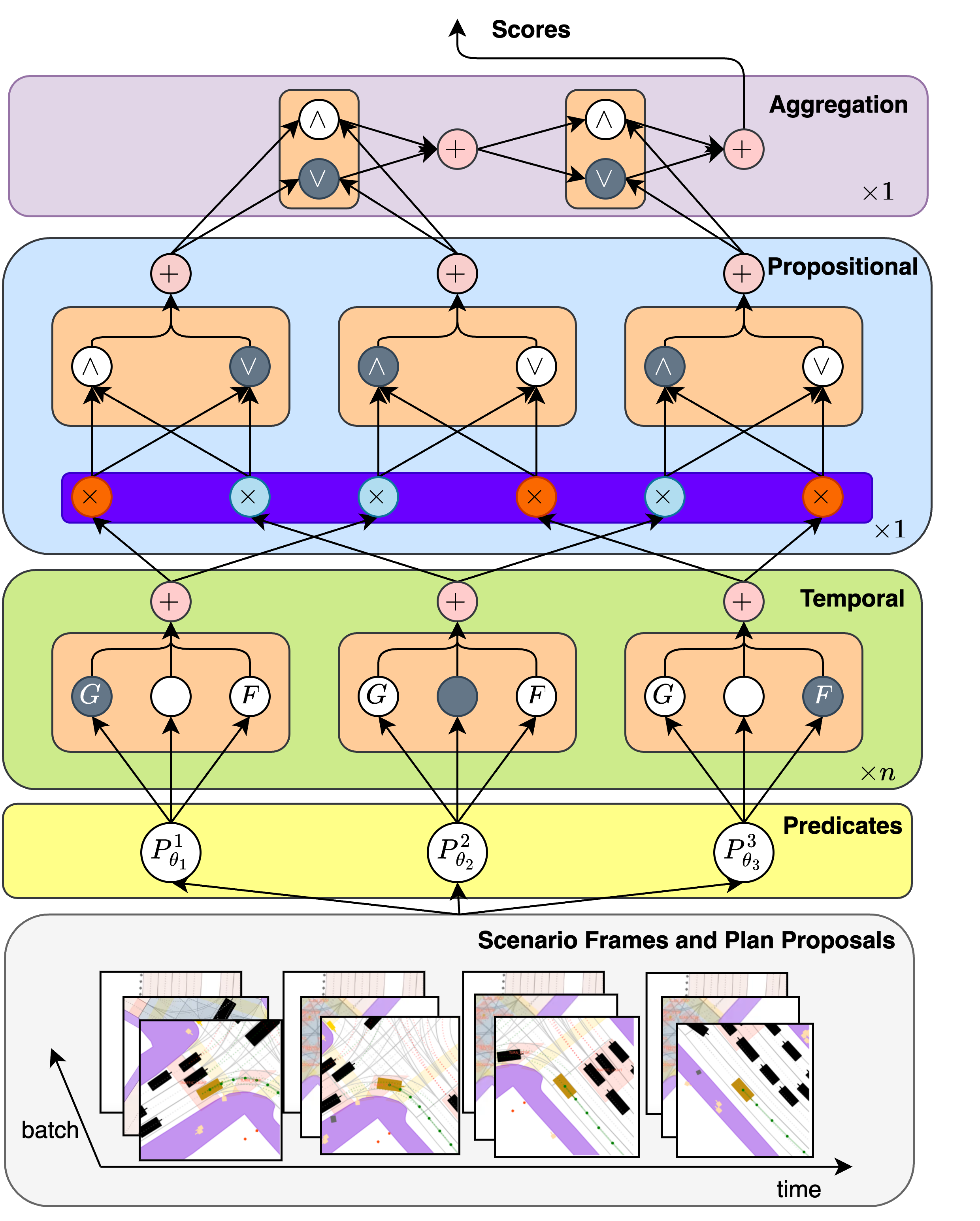}
    \caption{
        The logic structure $\bar{\mathcal{L}}$ consists of three types of layers: \textbf{Temporal, Propositional}, and \textbf{Aggregation}. The \textbf{Temporal} layer processes the initial predicates, applying temporal operators. The \textbf{Propositional} layer generates all possible pairs of predicates connected by logical operators. The \textbf{Aggregation} layer aggregates the output of the \textbf{Propositional} layer into one cluster by deciding the logic operator to connect neighboring clusters. \textbf{Temporal} layers can be stacked. a layer's formal definition is in Sec.~\ref{sec:layer_definition}.
        Two types of gates, the \selectiongate, and the \negationgate, are used to control the logic operators and the sign of the cluster inputs, respectively. Each clear circle ($\bigcirc$) in these gates represents a single value \emph{weight}. In the selection gate, the
        \protect\circled{graycirc} circle represents the operator with the largest weight, meaning the operator is selected. In the negation gate, the \protect\circled{bluecirc} circle represents the negation of the input (i.e., multiply with a negative number), while the \protect\circled{redcirc} circle represents the original input (i.e., a positive number). The gate implementation is described in Sec.~\ref{sec:gate_implementation}. Supposing we only consider one layer of Temporal layer ($n = 1$), and given a set of predicates $\PredicateSet = \{\PredicateWithIndex{1}, \PredicateWithIndex{2}, \PredicateWithIndex{3}\}$, $\PredicateWithIndex{2} \in \bar{\PredicateSet}$ and $\PredicateWithIndex{1}, \PredicateWithIndex{3} \in \dot{\PredicateSet}$, this learnable logic structure represents the logic formula $(\G \PredicateWithIndex{1} \lor \lnot \PredicateWithIndex{2}) \lor (\lnot \G \PredicateWithIndex{1} \land \F \PredicateWithIndex{3}) \lor (\lnot \PredicateWithIndex{2} \land \F \PredicateWithIndex{3})$
        This formula can be further reduced to $\PredicateWithIndex{2} \rightarrow (\G \PredicateWithIndex{1} \lor \F \PredicateWithIndex{3})$.
    }
    \label{fig:learnable_logic_structure}
\end{figure}

\subsubsection{Layer Definition}
\label{sec:layer_definition}
In the bottom block of Fig.~\ref{fig:learnable_logic_structure}, the frame batch is processed through $\PredicateSet$ and then passed through \textbf{Temporal} layers. Let $F = \{f_1, f_2, ..., f_T\}$ be a sequence of $T$ frames, where each frame represents a time point. Define $\PredicateSet = \{\PredicateWithIndex{1}, \PredicateWithIndex{2}, ..., \PredicateWithIndex{N}\}$ as a set of $N$ predicates. For each predicate $\PredicateWithIndex{i} \in \PredicateSet$ and each frame $f_t \in F$, we compute $X_i^t = \PredicateWithIndex{i}(f_t)$. Let $X_i^T = [x_i^1, x_i^2, ..., x_i^T]$ be the sequence of predicate values for predicate $\PredicateWithIndex{i}$ across all time steps. The output of the predicates is then defined as $\mathcal{X}^T = \{X_1^T, X_2^T, ..., X_N^T\}$, where each $X_i^T \in \mathbb{R}^T$ contains the predicate values computed over the entire time sequence for the $i$-th predicate. These frame sequences can be batched as shown in Fig.~\ref{fig:learnable_logic_structure}, to simplify the symbols, we only discuss the case with batch size one in the following parts.

The \textbf{Temporal} layer $\mathcal{T}$ operates on each element $X_i^T \in \mathcal{X}^T$, potentially applying a temporal operator. Formally, the output of $\mathcal{T}$ is $ O^T  = \mathcal{T}(\mathcal{X}^T) = \{o^T_1, o^T_2, \ldots, o^T_N\}$, where $o^T_i = \mathcal{T}(x_i^T) \in \{\mathbf{G} X_i^T, \mathbf{F} X_i^T, X_i^T\}$. Temporal layers can be stacked, allowing for the composition of temporal operators. For instance, with two stacked temporal layers, we could have $\mathcal{T}_2(\mathcal{T}_1(X_i^T)) = \mathbf{F}(\mathbf{G}(X_i^T))$. This composition allows for expressing more nuanced and complex temporal properties.

The \textbf{Propositional} logic layer $\mathcal{F}$ operates on its input set $O^T$, generating ${N \choose 2}$ clusters, each containing a combination of two inputs connected by a logical operator. The behavior of $\mathcal{F}$ is formalized as $\mathcal{F}(O^T) = O^P = \{o_1^P, o_2^P, \ldots, o_{N \choose 2}^P\}$, where $o_i^P = (\lnot) o_j^T \circ (\lnot) o_k^T$, $\circ \in \{\land, \lor\}$, and $O^T = \{o^T_1, o^T_2, \ldots, o^T_N\}$ is the output of a \textbf{Temporal} layer. Here, $j$ and $k$ represent the indices of the two different inputs being combined. Each input can be negated or unchanged when passing through a logic layer. We do not stack the \textbf{Propositional} layer as it would lead to exponential growth in the number of clusters; aggregating on one \textbf{Propositional} layer can represent any formula in the form of \eqref{eq:condition_action_pair} as we show in \ref{sec:logic_space_analysis}.

The \textbf{Aggregation} layer aggregates the output of the \textbf{Propositional} ($O^P$) layer into one cluster using logical connectives $\{\land, \lor\}$ to connect neighboring clusters. Formally, given the input $O^P$, the output of the \textbf{Aggregation} layer can be represented as $\mathcal{A}(O^P) = o^P_1 \circ_1 o^P_2 \circ_2 \ldots \circ_{{N \choose 2} - 1} o^P_{N \choose 2}$, where $\circ$ represents $\land$ or $\lor$.

A logic structure can be formally defined as
$
    \bar{\mathcal{L}}  = \mathcal{A}(\mathcal{F}(\mathcal{T}^{\times n}(\PredicateSet))),
$
where $n$ is the number of stacked Temporal layers.

\subsubsection{Gate Implementation}
\label{sec:gate_implementation}
The layers in Fig.~\ref{fig:learnable_logic_structure} are composed of selection gates and negation gates. Each $\bigcirc$ in these gates represents a weight $w \in \mathbb{R}$.

The \selectiongate~ acts as a soft attention mechanism to select between different operators across all layers (i.e., $\G, \F$ and identity operator in the Temporal layer, $\land, \lor$ in the \textbf{Propositional} and \textbf{Aggregation} layers), defined as:
\begin{equation}
    g_s(O) = \sigma([w_s^1 ; \cdots ; w_s^k]) \cdot [o_1 ; \cdots ; o_k]^T
\end{equation}
where $\sigma(\cdot)$ denotes the softmax function. For temporal operators in the \textbf{Temporal} layer, $O = [\rho(\G X_i^T), \rho(\F X_i^T), \rho(X_i^T)]^\top \in \mathbb{R}^3$ and $k=3$.
For logic operators in the \textbf{Propositional} and \textbf{Aggregation} layers, $O = [\rho(o_1 \land o_2), \rho(o_1 \lor o_2)]^\top \in \mathbb{R}^2$ and $k=2$. The evaluation function $\rho$ is defined in  \eqref{eq:temporal_operators} and \eqref{eq:fol_operators}. The selection gate blends operators using softmax to enable continuous gradient flow during training.

The \negationgate, given by $g_n = \tanh(w_{neg}) \cdot x$, with learnable parameters $w_{neg}$, controls the sign of cluster $x$ inputs to the Propositional layer.
\begin{equation}
    g_n = \tanh(w_{neg}) \cdot x, \quad w_{neg} \in \mathbb{R}, x \in \mathbb{R}
\end{equation}
The $\tanh$ function constrains the output to $[-1, 1]$. According to the quantitative semantics defined in \eqref{eq:fol_operators}, multiplying by a negative value negates the input. The gradient properties of $\tanh$ encourage $w_{neg}$ to converge towards either $-1$ or $1$ during training, effectively learning whether to negate the input.

\subsubsection{Logic Space Analysis}
\label{sec:logic_space_analysis}
Stacking \textbf{Temporal} layers monotonically increases the logic space. Creating a logic space that contains up to $n$ nested temporal operators can be easily achieved by stacking $n$ \textbf{Temporal} layers. For the \textbf{Propositional} layer, we assert that:

\begin{theorem}
    Aggregating the output from a single \textbf{Propositional} layer can represent any formula in the form of \eqref{eq:condition_action_pair}.
\end{theorem}

\begin{proof}
    Consider a \textbf{Propositional} layer with inputs ${\mathit{condition_1, \ldots, condition_n, action_1, \ldots, action_m}}$. For each $\mathit{action_i}$, combine:
    $$(\lnot \mathit{condition_1 \lor action_i) \land \ldots \land (\lnot condition_n \lor action_i})$$
    This is equivalent to $(\mathit{condition_1 \land \ldots \land condition_n) \rightarrow action_i}$. Aggregating for all $\mathit{action_i}$ yields:
    $$\bigodot_{i=1}^m \left((\mathit{condition_1 \land \ldots \land condition_n) \rightarrow action_i}\right)$$
    where $\bigodot$ is either $\land$ or $\lor$, matching the theorem's form.
\end{proof}

This theorem proves that a single \textbf{Propositional} layer is sufficient to express all condition-action pairs, justifying our computationally efficient design choice.

\subsubsection{Ensembling}
\label{sec:ensembling}

Given a set of logic structures $\{\bar{\mathcal{L}}_1, \bar{\mathcal{L}}_2, \ldots, \bar{\mathcal{L}}_k\}$, we can ensemble them by aggregating their outputs with an additional \textbf{Aggregation} layer. Formally,
\begin{equation}
    \begin{aligned}
        \bar{\mathcal{L}}_{\text{ensemble}} & = \mathcal{A}(\bar{\mathcal{L}}_1, \bar{\mathcal{L}}_2, \ldots, \bar{\mathcal{L}}_k)
    \end{aligned}
\end{equation}
In practice, we noticed that ensembling multiple logic structures can improve the robustness of the learned rules.

\subsubsection{Rule Extraction and Simplification}
\label{sec:rule_extraction_and_simplification}

The rule extraction and simplification process transforms the learned logic structure into interpretable condition-action pairs. To interpret the learned logic structure $\bar{\mathcal{L}}$, we extract its logic formula $\mathcal{L}$ by traversing the structure and selecting the most probable operators by concretizing the selection gates and negation gates. Given a selection gate $g_s$, let $\mathcal{C}(\cdot)$ denote the concretizing function:
\begin{equation}
    \begin{aligned}
        \mathcal{C}(g_s) & = \argmax_{op} w^{op}_s, \quad op \in \begin{cases}
                                                                     \{\mathbf{G}, \mathbf{F}, \text{id}\}, \\
                                                                     \{\land, \lor\},
                                                                 \end{cases}
    \end{aligned}
    \label{eq:selection_gate_concretizing}
\end{equation}
where $w^{op}_s$ represents the weights associated with each operator in the selection gate.
For the negation gate $g_n$, the concretizing is determined by the sign of the weight:
\begin{equation}
    \begin{aligned}
        \mathcal{C}(g_n) & = \begin{cases}
                                 \text{negation}, & \text{if } w_{neg} < 0 \\
                                 \text{original}, & \text{otherwise}
                             \end{cases}
    \end{aligned}
    \label{eq:negation_gate_concretizing}
\end{equation}
A concrete example is shown in Fig.~\ref{fig:learnable_logic_structure} by iteratively applying \eqref{eq:selection_gate_concretizing} and \eqref{eq:negation_gate_concretizing} to the selection and negation gates. For an ensemble logic structure, we only need to apply the same rule further for the additional Aggregation layer.

The extracted formula is then simplified to a set of condition-action pairs in the form of \eqref{eq:condition_action_pair}.
We apply the Quine-McCluskey algorithm~\cite{McCluskey1956MinimizationOB} to simplify the extracted formula.
Such simplification removes redundant cluster (e.g., $\land (\PredicateWithIndex{1} \lor \lnot \PredicateWithIndex{1})$).
The resulting formula is then converted to conjunctive normal form:
$
    \bigwedge_{i=1}^n \left( \bigvee_{j=1}^m \bar{P}_j \lor \bigvee_{k=1}^l \dot{P}_k \right),
$
where $\bar{P}_j$ and $\dot{P}_k$ represent condition and action predicates, respectively. By double negating the condition predicates and applying De Morgan's laws, this formula can be further simplified to condition-action pairs in the form of \eqref{eq:condition_action_pair}.

\subsection{Learning From Demonstration}
\label{sec:learning_from_single_class_demonstration}

Most demonstration datasets consist solely of examples representing ideal driving behaviors. Our goal is to learn the logic structure $\bar{\mathcal{L}}^*$ \footnote{$\bar{\mathcal{L}}$ means the learnable logic structure, while $\mathcal{L}$ means a concrete logic formula defined by \eqref{eq:logic_formula_syntax}. } and predicate parameters $\boldsymbol{\theta}^*$ that grade demonstrations with the highest score and unseen behaviors with lower scores. Directly optimizing on \eqref{eq:objective} would simply make the score $\bar{\mathcal{L}}(S; \boldsymbol{\theta})$ to be 1 for all demonstrations.
The optimizer could find ``shortcuts'' and result in learning trivial formulas like $\PredicateWithIndex{1} \lor \lnot \PredicateWithIndex{1}$, which are always true and do not provide meaningful rules.
To address this issue, we introduce two regularization terms for $\boldsymbol{\theta}$ and $\bar{\mathcal{L}}$. The full learning algorithm is presented in Algorithm \ref{algo:train-with-regularization}.
\begin{algorithm}[ht]
    \caption{Training with Regularization}
    \label{algo:train-with-regularization}
    \KwIn{Dataset $\mathcal{D}^+$, initial logic structure $\bar{\mathcal{L}}$, initial parameters $\boldsymbol{\theta}$, learning rates $\alpha$, $\beta$, max weight $w_{\max}$, batch size $B$, number of epochs $E$}
    \KwOut{Optimized $\bar{\mathcal{L}}^*$ and $\boldsymbol{\theta}^*$}

    \SetKwFunction{FUpdate}{Update}
    \SetKwFunction{FRegularize}{Regularize}

    \SetKwProg{Fn}{Function}{:}{}
    \Fn{\FUpdate{$\bar{\mathcal{L}}, \boldsymbol{\theta}, \mathcal{J}, \gamma$}}{
        Update $\boldsymbol{\theta}$ with $\nabla_{\boldsymbol{\theta}} \mathcal{J}$\;
        Update $\bar{\mathcal{L}}$ (i.e., gate weights) with $\nabla_{\bar{\mathcal{L}}} \mathcal{J}$\;
        \Return $\bar{\mathcal{L}}, \boldsymbol{\theta}$\;
    }

    \Fn{\FRegularize{$\bar{\mathcal{L}}, \boldsymbol{\theta}, \alpha, \beta, w_{\max}$}}{
        $\boldsymbol{\theta} \leftarrow \boldsymbol{\theta} - \alpha \cdot \text{sign}(\frac{\partial \bar{\mathcal{L}}}{\partial \boldsymbol{\theta}})$\;
        \ForEach{Aggregation layer in $\bar{\mathcal{L}}$}{
            $w_{\land} \leftarrow \min(w_{\land} + \beta, w_{\max})$\;
        }
        \Return $\bar{\mathcal{L}}, \boldsymbol{\theta}$\;
    }

    \For{epoch $\leftarrow 1$ \KwTo $E$}{
        \For{each batch $\{S_1, \ldots, S_B\} \sim \mathcal{D}^+$}{
            $\mathcal{J} \leftarrow \frac{1}{B}\sum_{i=1}^B \bar{\mathcal{L}}(S_i; \boldsymbol{\theta})$\;
            $\bar{\mathcal{L}}, \boldsymbol{\theta} \leftarrow$ \FUpdate{$\bar{\mathcal{L}}, \boldsymbol{\theta}, \mathcal{J}, \gamma$}\;
            $\bar{\mathcal{L}}, \boldsymbol{\theta} \leftarrow$ \FRegularize{$\bar{\mathcal{L}}, \boldsymbol{\theta}, \alpha, \beta, w_{\max}$}\;
        }
    }
    \Return $\bar{\mathcal{L}}^* \leftarrow \bar{\mathcal{L}}$, $\boldsymbol{\theta}^* \leftarrow \boldsymbol{\theta}$
\end{algorithm}

The learning system evaluates driving behaviors in a state space $S = \{(\mathcal{E}_t, \tau_t)\}_{t=0}^T$, where behaviors with $\bar{\mathcal{L}}(S; \boldsymbol{\theta}) > 0$ are considered acceptable. Without constraints, the system might learn shortcuts that accept almost any behavior. We prevent this through two types of regularization:

\subsubsection{Preventing Shortcut Learning Through Regularization}

The learning system evaluates driving behaviors in a state space $S = \{(\mathcal{E}_t, \tau_t)\}_{t=0}^T$, where behaviors with $\bar{\mathcal{L}}(S; \boldsymbol{\theta}) > 0$ are considered acceptable. Without constraints, the system might learn shortcuts that accept almost any behavior.
Consider a simple example where we want to learn rules for safe lane changes. Without regularization, the system might learn the rule:
$\text{InLane} \lor \text{SafeDistance}$
that accepts behaviors in which the car is either in a lane OR maintains safe distance, which is clearly too permissive. We prevent such shortcuts through two types of regularization: First, we gradually tighten the parameters $\boldsymbol{\theta}$ of each rule using:
$
    \boldsymbol{\theta} = \boldsymbol{\theta} - \alpha \cdot \text{sign}\left(\frac{\partial \bar{\mathcal{L}}}{\partial \boldsymbol{\theta}}\right).
$
In our example, this might gradually increase the required safe distance threshold from 1 meter to a more reasonable 2 meters. Second, when combining rules, we encourage the use of AND ($\land$) over OR ($\lor$) operations using:
$
    w_{op}^{\land} = \min(w_{op}^{\land} + \beta, w_{\max})
$
This helps the system learn more appropriate rules like:
$\text{InLane} \land \text{SafeDistance}$
requiring both conditions to be met for a safe lane change.

Importantly, these constraints only eliminate spurious shortcuts - any rule structure or parameter that genuinely captures important driving behavior will remain intact, as it will be consistently reinforced by the demonstration data despite the regularization pressure.
\section{Experiments}

\subsection{General Experiment Setup}
We consider 63 condition predicates and 20 action predicates in the predicate set $\PredicateSet$.
All predicates are categorized into four types: (1) safety-related (e.g., will the plan result in obstacle collision?) (2) comfort-related (e.g., is the acceleration in a reasonable range?); (3) efficiency-related (e.g., is the car is in a slow lane?); (4) environment-related (e.g., is the curvature of the current lane too high?).
Implementing the predicate set $\PredicateSet$ involves LLM-aided design \cite{Chen2022CodeTCG}, but the final predicates were manually checked to ensure correctness and sensibility.

We used all the NuPlan demonstrations in Singapore, Pittsburgh, and Boston, and part of the demonstrations in Las Vegas, to train the logic structure. The scenarios are grouped into 9 types (the first column in Table \ref{tab: closed-loop-planning-performance}). The dataset is split into 90\% training and 10\% validation. The validation set is used for hyperparameter tuning and early stopping. Following PDM \cite{Dauner2023CORL}, we use the log replay to obtain the 4-second future trajectories of other agents in both training and evaluation.

We trained an ensemble $\LogicStucture_{\mathit{ensemble}}$ with 10 different $\LogicStucture$. Each $\LogicStucture$ has 2 layers of the Temporal layer. The single class regularization parameters $\alpha$ and $\beta$ are set to $10^{-5}$ and $10^{-3}$, respectively. The learning rate is set to $10^{-4}$, and optimized with the Adam optimizer. The batch size is set to 32. The training process is stopped when the validation loss does not decrease for 10 epochs. We used 15 trajectory candidates for evaluation in closed-loop simulation.

\subsection{Case Study}
\subsubsection{Logic Rules Discovered}
\label{sec:logic_rules_discovered}

Fig.~\ref{fig:motivation_example} shows three of the rules learned by the grading logic network. Given the motion plan proposed by PDM \cite{Dauner2023CORL}, the grading logic network assigns a score to each plan based on the learned rules. Fig.~\ref{fig:case-study-example} shows why the \textcolor[HTML]{2278B4}{blue}, \textcolor[HTML]{D52928}{red} and \textcolor[HTML]{F98217}{orange} plans receive lower scores. The blue plan receives a low score for going beyond the drivable area (the gray area), which violates the rule $\top \rightarrow \mathbf{G} \mathit{InDrivable}$. The red plan is penalized for exceeding comfort constraints (the blue dashed line) on lateral acceleration, which violates the rule $\mathbf{G} \mathit{SafeTTC} \rightarrow \mathbf{G} \mathit{Comfortable}$. The orange plan receives a lower score for exceeding the speed limit (the red dashed line) when not overtaking another vehicle, which violates the rule $\lnot \mathit{SpeedLimit} \rightarrow \mathbf{F} \mathit{Overtaking}$. These rules are almost always discovered in our experiment. One exception is that the rule $\lnot \mathit{SpeedLimit} \rightarrow \F \mathit{Overtaking}$ could sometimes devolve to  $\mathit{SpeedLimit}$ or $\G \mathit{SpeedLimit}$   in scenarios where overtaking happens rarely (e.g., following other cars). More detailed case studies are provided on our website.

\begin{figure}[h!]
    \centering
    \includegraphics[width=\linewidth]{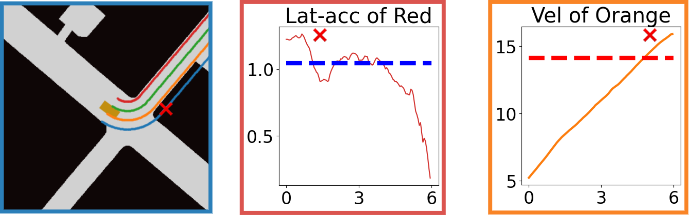}
    \caption{
        Case study on discovered rules. While our method evaluates 15 trajectory candidates in practice, we show only 4 representative trajectories here for visual clarity. From left to right, these sub-figures explain why the blue, red, and orange plans received lower scores. The blue plan violates the drivable area rule, the red plan exceeds comfort constraints on lateral acceleration, and the orange plan breaks the speed limit without overtaking context.
    }
    \label{fig:case-study-example}
\end{figure}

\subsubsection{Parameter Optimization}

$\mathit{Comfortable}_{\boldsymbol{\theta}}$ is a predicate measure if a motion plan is comfortable, which has parameters $\boldsymbol{\theta} = [\theta_{a_{f}}, \theta_{a_{b}}, \theta_{a_{l}}, \theta_{a_{r}}]$, representing thresholds for acceleration in forward, backward, left, and right directions. It can be defined as:
\begin{align}
    \mathit{Comfortable}_{\boldsymbol{\theta}} := \tanh(\min_{\substack{i \in \{a_{f}, a_{b}, a_{l}, a_{r}\}}} \{\theta_i - i(\tau)\})
    \label{eq:comfortable_predicate}
\end{align}
where $i(\tau)$ represents the maximum acceleration in the plan $\tau$ for each respective direction.
Only when all the accelerations are below their thresholds, is the predicate evaluated as positive. The $\tanh$ function is used to ensure the output is in $[-1, 1]$. All threshold parameters in \eqref{eq:comfortable_predicate} are differentiable and can be learned from the training demonstration data.
The learned parameters $\boldsymbol{\theta}$ are shown in Table \ref{tab: parameter-optimization}.

\begin{table}[!htp]\centering
    \caption{Case Study on Learned Parameters of $\mathit{Comfortable}_{\boldsymbol{\theta}}$}\label{tab: parameter-optimization}
    \scriptsize
    \begin{tabular}{l|cccc}\toprule
                          & $\theta_{a_{f}}$ & $\theta_{a_{b}}$ & $\theta_{a_{l}}$ & $\theta_{a_{r}}$ \\\midrule
        \textbf{Standard} & 1.23             & 1.13             & 0.98             & 0.98             \\
        \textbf{Learned}  & 1.1  $\pm$ 0.21  & 1.045 $\pm$ 0.12 & 0.9 $\pm$ 0.11   & 0.95 $\pm$ 0.51  \\
        \bottomrule
    \end{tabular}
    \begin{tablenotes}
        \item The learned parameters are shown in the format of mean $\pm$ standard deviation computed from 5 independent runs.
    \end{tablenotes}
\end{table}

The \textbf{Standard} parameters of comfortable acceleration are provided in \cite{deWinkel2023Standards} from a user study.
We observe the \textbf{Learned} parameters are close to \textbf{Standard}. In practice, it is noticeable that the acceleration thresholds are hard to estimate. However, the learned parameters are comparable to the standard values by learning from data, which indicates the proposed method can learn the parameters characterizing the demonstration data well. The notable standard deviations in Table~\ref{tab: parameter-optimization} reflect inherent variations in driving behavior, from right turns showing higher variation than left turns, to forward acceleration varying more than deceleration. These variations persist despite analyzing over 100 hours of driving data, suggesting they stem from inherent differences in driving environments (e.g., Boston's narrow streets versus Las Vegas's wide boulevards) rather than insufficient data. Potentially, better scenario-based classification could address these variations, but we leave this for future work.

\subsection{Evaluation in Closed-loop Planning}

The learned rules are evaluated in a closed-loop planning system using NuPlan under two settings: Closed-Loop Score Non-Reactive (CLS-NR), where other agents follow recorded trajectories, and Closed-Loop Score Reactive (CLS-R), where agents are controlled by an Intelligent Driver Model (IDM) policy. Performance is measured using normalized scores (0-1) across safety (collision avoidance, following distances), rule compliance (lane keeping, speed limits), progress along the route, and comfort (acceleration, jerk), assessing the rules' ability to balance driving requirements in both predictable and interactive scenarios. The results are shown in Table \ref{tab: closed-loop-planning-performance}. Here, we compare our learned Scoring Logic Network (SLN) with the Expert Rules (ER) in PDM \cite{Dauner2023CORL} and the Neural Critic (NC) \cite{jiang2022efficient}. The last two rows show the overall performance on ``All'' the scenarios and the Val14 split originally used for evaluating the ER \cite{Dauner2023CORL}.

\begin{table}[!htp]\centering
    \caption{Closed-Loop Planning Performance}
    \label{tab: closed-loop-planning-performance}
    \scriptsize
    \begin{tabular}{l@{\hspace{0.8em}}|c@{\hspace{0.8em}}c@{\hspace{0.8em}}c@{\hspace{0.8em}}|c@{\hspace{0.8em}}c@{\hspace{0.8em}}c@{\hspace{0.8em}}|c@{\hspace{0.8em}}c@{\hspace{0.8em}}c}\toprule
                    & \multicolumn{3}{c}{Rule Complexity} & \multicolumn{3}{c}{CLS-NR $\uparrow$} & \multicolumn{3}{c}{CLS-R $\uparrow$}                                                                               \\\cmidrule{2-10}
                    & $|\bar{\mathcal{P}}|$               & $|\dot{\mathcal{P}}|$                 & \#. Rules                            & ER            & NC   & SLN           & ER            & NC   & SLN           \\\midrule
        Change Lane & 20                                  & \multirow{9}{*}{20}                   & 24                                   & 0.89          & 0.79 & \textbf{0.92} & 0.88          & 0.77 & \textbf{0.91} \\
        Following   & 10                                  &                                       & 16                                   & \textbf{0.94} & 0.88 & 0.91          & \textbf{0.96} & 0.90 & \textbf{0.96} \\
        Near Static & 10                                  &                                       & 19                                   & 0.87          & 0.77 & \textbf{0.93} & 0.87          & 0.85 & \textbf{0.87} \\
        Near VRU    & 10                                  &                                       & 17                                   & 0.87          & 0.81 & \textbf{0.93} & \textbf{0.89} & 0.75 & 0.87          \\
        Turn        & 20                                  &                                       & 31                                   & 0.89          & 0.82 & \textbf{0.91} & 0.89          & 0.71 & \textbf{0.91} \\
        Stopping    & 10                                  &                                       & 13                                   & \textbf{0.90} & 0.77 & \textbf{0.90} & 0.92          & 0.85 & \textbf{0.93} \\
        Starting    & 20                                  &                                       & 27                                   & \textbf{0.91} & 0.85 & 0.89          & 0.88          & 0.84 & \textbf{0.90} \\
        Stationary  & 20                                  &                                       & 21                                   & \textbf{0.94} & 0.73 & \textbf{0.94} & 0.96          & 0.81 & \textbf{0.97} \\
        Traversing  & 20                                  &                                       & 29                                   & 0.87          & 0.71 & \textbf{0.90} & 0.89          & 0.70 & \textbf{0.90} \\ \midrule
        All         & \multirow{2}{*}{63}                 & \multirow{2}{*}{20}                   & \multirow{2}{*}{124}                 & 0.90          & 0.79 & \textbf{0.92} & 0.91          & 0.81 & \textbf{0.93} \\
        Val14       &                                     &                                       &                                      & 0.93          & 0.81 & \textbf{0.94} & 0.92          & 0.83 & \textbf{0.93} \\
        \bottomrule
    \end{tabular}
    \begin{tablenotes}
        \item The videos and explanation of all the scenarios are available \href{https://xiong.zikang.me/FLoRA/#scenarios}{online}.
        \item The proposal approach we used here is from PDM \cite{Dauner2023CORL}.
    \end{tablenotes}
\end{table}

The complexity of rules is measured by the number of action predicates $|\bar{\mathcal{P}}|$, the number of condition predicates $|\dot{\mathcal{P}}|$, and the total number of extracted action-condition pairs. The detailed results are shown in Table \ref{tab: closed-loop-planning-performance}. In the All and Val14 splits, we used all of our 63 condition predicates and 20 action predicates. In the CLS-NR setting, SLN outperforms ER in 6 out of 9 scenario types, ties in 2, and underperforms in 1, while in CLS-R, it surpasses ER in 5 scenarios, ties in 2, and falls short in 2. Notably, SLN achieves this performance \textit{without the complex relationship modeling or extensive parameter tuning} required by ER. Compared to the NC, SLN exhibits superior performance across all 9 scenario types in both settings, offering substantial improvement coupled with interpretability. Overall (in the last two rows), SLN consistently outperforms both ER and NC for the ``All'' scenarios and the ``Val14'' \cite{Dauner2023CORL} split is used to evaluate ER, in both reactive and non-reactive settings. Following NuPlan's log frequency \cite{Karnchanachari2024TowardsLP}, we run the closed-loop planning system at 20 Hz. To conserve computational resources, we also evaluated performance at 10 Hz and 5 Hz. Results show minimal performance impact, with scores the same on the ``All'' scenarios and minor degradation (<0.2, CLS-NR: 0.93, CLS-R: 0.91) on the ``Val14'' split for both reduced frequencies.

\subsection{Ablation Studies}
\label{sec:ablation-studies}

\subsubsection{Proposal Approach}
\label{sec:proposal_approach}
A robust rule should above all be able to filter out undesirable plans for different proposal approaches.
We evaluate the learned rules on different proposal approaches, including PDM \cite{Dauner2023CORL}, AT-Sampler \cite{jiang2022efficient}, Hybrid-PDM \cite{Dauner2023CORL}, and ML-Prop \cite{scheel2022urban}. For both rule-based and learning-based proposal approaches, our learned rules demonstrate superior performance in the closed-loop planning system, as shown by the data presented in Table \ref{tab: proposal-approach-ablation}.
\begin{table}[h!]\centering
    \caption{Proposal Approach Ablation on ``All''}\label{tab: proposal-approach-ablation}
    \scriptsize
    \begin{tabular}{l|ccc|cccr}\toprule
                                             & \multicolumn{3}{c|}{CLS-NR $\uparrow$} & \multicolumn{3}{c}{CLS-R $\uparrow$}                                               \\\cmidrule{2-7}
                                             & ER                                     & NC                                   & SLN (ours)    & ER   & NC   & SLN (ours)    \\\midrule
        PDM \cite{Dauner2023CORL}            & 0.90                                   & 0.79                                 & \textbf{0.92} & 0.91 & 0.81 & \textbf{0.93} \\
        AT-Sampler \cite{jiang2022efficient} & 0.83                                   & 0.81                                 & \textbf{0.89} & 0.84 & 0.80 & \textbf{0.90} \\ \midrule
        Hybrid-PDM   \cite{Dauner2023CORL}   & 0.90                                   & 0.79                                 & \textbf{0.92} & 0.91 & 0.79 & \textbf{0.92} \\
        ML-Prop$^*$ \cite{scheel2022urban}   & 0.81                                   & 0.75                                 & \textbf{0.86} & 0.82 & 0.76 & \textbf{0.85} \\
        \bottomrule
    \end{tabular}
    \begin{tablenotes}
        \item $^*$ \cite{scheel2022urban} introduced an deterministic policy. We generate one initial proposal with this policy and then add lateral deviations to generate multiple parallel proposals.
    \end{tablenotes}
\end{table}

\subsubsection{Number of Proposal Candidates}
\label{sec:number_of_proposal_candidates}

\begin{wrapfigure}{r}{0.25\textwidth}
    \vspace{-0.5cm}
    \centering
    \includegraphics[width=\linewidth]{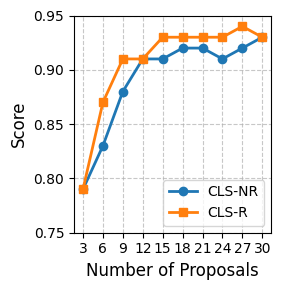}
    \caption{Proposal Number Ablation}
    \label{fig:ablation_candidates}
\end{wrapfigure}

As shown in \ref{fig:ablation_candidates}, the performance of both CLS-NR and CLS-R improves significantly when increasing proposals from 3 to 15 (CLS-NR: $0.81 \rightarrow 0.92$, CLS-R: $0.79 \rightarrow 0.93$), but plateaus beyond 15 proposals, showing only minimal gains up to 30. Therefore, 15 proposals were used in the main experiments.

\subsubsection{Regularization Hyperparameters}
\label{sec:regularization_hyperparameters}

We conducted an ablation study to examine the impact of regularization hyperparameters $\alpha$ and $\beta$ on learning performance.
\begin{wraptable}{l}{0.6\linewidth}
    \centering
    \begin{threeparttable}
        \caption{Regularization Ablation}\label{tab:ablation-on-regularization}
        \centering
        \scriptsize
        \begin{tabular}{@{}ccc|ccc@{}}\toprule
            $\alpha$  & Conv & Triv & $\beta$   & Conv & Triv \\\midrule
            $10^{-4}$ & 99   & 0    & $10^{-2}$ & 131  & 0    \\
            $10^{-5}$ & 37   & 0    & $10^{-3}$ & 37   & 0    \\
            $10^{-6}$ & 24   & 0.4  & $10^{-4}$ & 27   & 0.5  \\
            $10^{-7}$ & 13   & 0.9  & $10^{-5}$ & 12   & 1    \\
            0         & 11   & 1    & 0         & 11   & 1    \\
            \bottomrule
        \end{tabular}
    \end{threeparttable}
\end{wraptable}
``Conv'' means the rate of convergence measured by the number of epochs before the validation loss stops decreasing for 10 epochs. ``Triv'' is the ratio of learned rules that are trivial (e.g., $\top \rightarrow \top$) across 10 runs in the ``Following'' scenario.
While varying $\alpha$ ($\beta$), the value of $\beta$ ($\alpha$) is fixed to $10^{-3}$ ($10^{-5}$).
When $\alpha = 0$ or $\beta = 0$ (i.e. no regularization), the model converges quickly (11 epochs) but learns only trivial rules.
Larger $\alpha$ ($10^{-7}$ to $10^{-4}$) and $\beta$ ($10^{-5}$ to $10^{-2}$) values increase convergence time but reduce trivial rules. We found $\alpha = 10^{-5}$ and $\beta = 10^{-3}$ an optimal choice, achieving convergence in 37 epochs with no trivial rules, balancing training efficiency and rule quality.

\section{Conclusion, Limitations, and Future Work}

This paper introduces FLoRA, a framework for learning interpretable scoring rules expressed in temporal logic for autonomous driving planning systems. FLoRA addresses key challenges by developing a learnable logic structure to capture nuanced relationships among driving predicates; proposing a data-driven method to optimize rule structure and parameters from demonstrations; and presenting an optimization framework for learning from driving demonstration data. Experimental results on the NuPlan dataset show that FLoRA outperforms both expert-crafted rules and neural network approaches across various scenarios, with different proposer types. This work thus represents a significant step towards creating more adaptable, interpretable, and effective scoring mechanisms for autonomous driving systems, bridging the gap between data-driven approaches and rule-based systems. A current limitation of FLoRA is that it can only discover relationships between provided predicates, which are manually designed. Future work will focus on extending FLoRA with LLM-aided predicate discovery and incorporating accident data as they become available through real-world autonomous driving deployments.

\bibliographystyle{IEEEtran}
\bibliography{ref}

\end{document}